\documentclass{article}

\usepackage{arxiv}
\usepackage{amssymb}
\usepackage[utf8]{inputenc} 
\usepackage[T1]{fontenc}    
\usepackage{hyperref}       
\usepackage{url}            
\usepackage{booktabs}       
\usepackage{amsfonts}       
\usepackage{nicefrac}       
\usepackage{microtype}      
\usepackage[linesnumbered,ruled]{algorithm2e}
\usepackage{amsmath}
\usepackage{float}
\usepackage{lipsum}
\usepackage{graphicx}
\usepackage{amsmath, amsthm, amssymb}
\newtheorem{prop}{Proposition}

\newtheorem*{remark}{Remark}

\title{A Deep Reinforcement Learning Architecture for Multi-stage Optimal Control}

\author{
  Yuguang Yang \\
  Johns Hopkins University \\
  \texttt{yyang60@jhu.edu} \\
}

\begin{document}
\maketitle

\begin{abstract}
Deep reinforcement learning for high dimensional, hierarchical control tasks usually requires the use of complex neural networks as functional approximators, which can lead to inefficiency,  instability and even divergence in the training process. Here, we introduce stacked deep Q learning (SDQL), a flexible modularized deep reinforcement learning architecture, that can enable finding of optimal control policy of control tasks consisting of multiple linear stages in a stable and efficient way. SDQL exploits the linear stage structure by approximating the Q function via a collection of deep Q sub-networks stacking along an axis marking the stage-wise progress of the whole task. By back-propagating the learned state values from later stages to earlier stages, all sub-networks co-adapt to maximize the total reward of the whole task, although each sub-network is responsible for learning optimal control policy for its own stage. This modularized architecture offers considerable flexibility in terms of environment and policy modeling, as it allows choices of different state spaces, action spaces, reward structures, and Q networks for each stage, Further, the backward stage-wise training procedure of SDQL can offers additional transparency, stability, and flexibility to the training process, thus facilitating model fine-tuning and hyper-parameter search. We demonstrate that SDQL is capable of learning competitive strategies for problems with characteristics of high-dimensional state space, heterogeneous action space(both discrete and continuous), multiple scales, and sparse and delayed rewards.
\end{abstract}

\keywords{Multi-stage optimal control \and Deep reinforcement learning}

\section{Introduction}

We consider optimal control tasks that consists of multiple linear stages. Such multi-stage optimal control problems arise from a broad range of areas\cite{Wallace2005ApplicationsProgramming, Kroemer2015TowardsTasks, Sutton2015ReinforcementEd}, including robotics, inventory planning, production planning, transportation and logistics, financial management, and others. The stage structure of a task need not refer to the passage of time, but can corresponds to external specifications or the results of subjective modeling strategy. For example, in quantitative investment, personal investment activities naturally consist of  several stages based on the investor’s age and future obligations, with each stage having different preference for risk and return. Separately, a task that involves multiple agents’ sequential cooperation (i.e., the second agent continue with where the first agent has left, but not concurrently working on the same task) has each stage associated with each individual agent. The stage structure can also be a result of a divide-and-conquer strategy that decompose a complex, temporally extended task into multiple simple ones.  For example, a robot's gripping task can be divided into a first stage of  moving the bulk part of the gripper near the object, and a second stage of fine adjustment of fingers to pick up the object. Finally, the stage structure can also be introduced as a modeling strategy to exploit model characteristics(i.e., state/action space and reward structure) at different stages. For example, in financial engineering, the optimal hedging strategies for options can be decomposed into a near-term stage and a long-term stage to exploit different dominant dynamics in the two regimes. 

Here, we aim to develop a general deep reinforcement learning architecture and training algorithm that exploit the stage structure and facilitate the learning process. We adopt a general multi-stage task setting, where different stages share a subset of state space, but allow to have different action spaces, environment model dynamics, and reward signals. We propose a learning architecture, called Stacked Deep Q Learning (SDQL), where we approximate the optimal control policy and its value function for a multi-stage task via an architecture consisting a series of deep Q sub-networks stacked together [\ref{fig:fig1}]. Each sub-network is exclusively responsible for learning and control at one stage; thus in general they are simpler and smaller neural networks that can be trained in a stably and efficient way. To ensure that all sub-networks are co-adapting to maximize the total reward of the whole task, instead of its own stage, we require that each sub-network is trained to maximize the cumulative rewards in its own stage plus the reward triggering a stage transition (i.e., finish of a stage and start of the next stage). The stage-transition reward is provided by the value function of sub-network responsible for the next stage. The corresponding training algorithm will involve a stage-wise training process from later stages to earlier stages and state value back-propagation; in this way all sub-networks thus are trained to  maximize the total reward across all stages instead of the cumulative rewards in its own stage. 

The idea of stacking multiple neural networks and training from back to front are mainly inspired by the backward induction in dynamic programming, especially its application in finite-horizon Markov decision process\cite{Puterman2014MarkovProgramming}. Using multiple Q networks in SDQL offers a number of advantages over using a complex neural network for end-to-end learning\cite{Levine2015End-to-EndPolicies}. First, the use of multiple simpler Q networks severs as a form of regulation on the network structure. Simpler network can prevent over-fitting and usually can be trained in a more efficient and stable way since each sub-network is usually responsible for learning at a smaller state space and action space. The stage-wise training also allows easier monitoring of the training process, thus adding transparency in the training process. Second, SDQL allows stacking different types of Q learning sub-networks, such as the Deep Q Network (DQN) \cite{Mnih2015Human-levelLearningb} and actor-critic architecture like Deep Deterministic Policy Gradient (DDPG) network \cite{Lillicrap2015ContinuousLearning} together), which enables flexibility in modeling  environment dynamics (homogeneous and heterogeneous) and action space (discrete and continuous) at different stage of the task. Thirdly, the use of multiple neural network also facilitates transfer leaning.  For example, a sub-network that is responsible for moving the bulk part of a gripper in an early stage can be composed with other sub-networks that is responsible for fine  movements in later stages. The reuse and combination of different sub-networks can produce control policies for different tasks. Finally, the backward stage-wise training process resembles reverse curriculum learning \cite{Bengio2009CurriculumLearning} where the agent schedules the learning process to start with easier states (e.g., states near the goal state or states the agent can achieve the goal with high probability), gradually increasing the difficulties. By stacking multiple layers together and train backward, control tasks with delayed and sparse rewards can be naturally addressed by SDQL. 

Our main contribution is the introduction of SDQL, an abstract modeling framework and architecture that allows integrating different Q learning network to tackle multi-stage optimal control tasks. We provide a stage-wise backward training algorithm and demonstrate its use in several control tasks with characteristics of high-dimensional state space, heterogeneous action space(both discrete and continuous), multiple scales, and sparse and delayed rewards.

\section{Theory}
\subsection{Preliminary: Markov decision process (MDP)}
Consider an infinite-horizon Markov decision process (MDP) represented by a tuple $(\mathcal{S}, \mathcal{A}, \mathcal{P})$
where $\mathcal{S}$ is the state space, $\mathcal{A}$ is the action space, and $\mathcal{P} = \{p(s'|s, a), s,s'\in \mathcal{S}, a\in \mathcal{A}\}$ is the state transition probability. 
The goal is compute an optimal control policy $\pi^*: \mathcal{S}\to \mathcal{A}$ such that the expected  total reward in the process
\begin{equation}J = \mathbf{E}[\sum_{t=0}^\infty \gamma^t R(s_{t+1}, a_t)]\end{equation}
is maximized, where $R(s,a):\mathcal{S}\times\mathcal{A}\to\mathbf{R}$ is the one-step reward function and $\gamma$ is the discount factor. 
In the Q learning framework, we define optimal state-action value function (also known as the $Q^*$ function) by
\begin{equation}
Q^*(s, a) = \mathbf{E}^\pi[\sum_{t=0}^\infty \gamma^t R(s_{t+1}, a_t)|s_0 = s, a_0 = a, \pi^*]
\end{equation}
which is the expected sum of rewards along the process by following the optimal policy $\pi^*$, after visiting state $s_0$ and taking action $a_0$. With $Q^*$ , the optimal policy $\pi^*$ is given by $\pi^* = \arg\max_a Q^*(s,a)$

\subsection{MDP with stage structure}

We introduce a stage structure into a goal-oriented control tasks in the following way. For a control task with $N$ linear stages, we define $N$ state space subsets $\mathcal{S}_1,..., \mathcal{S}_N$of $\mathcal{S}$  satisfying $$\mathcal{S}_1\subset \mathcal{S}_2 \subset \mathcal{S}_3 \cdots \mathcal{S}_N = \mathcal{S}.$$
Given an initial state $s_0$, we define its associated stage index by $I(s_0) = min_{i} \{i| s_0 \in \mathcal{S}_i\}$. As we evolve the system state forward $s_0\to s_1 \to \cdots$, the change stage index of a state $s_i$ is determined by a stage-transition event. Let $I(s_t) = i$. We say the transition tuple $(s_t, a_t, s_{t+1})$ triggers state-transition event such that $I(s_{t+1} = i + 1$ if $s_{t+1} \in \mathcal{S}_{i+1} - \mathcal{S}_i$ (i.e., the system moves out of $\mathcal{S}_{i}$ and enters $\mathcal{S}_{i+1}$).

In SDQL, we consider infinite horizon episodic goal-oriented tasks that have terminal states $\mathcal{S}_T \subset \mathcal{S}_N$. We assume the system state starting with $\mathcal{S}_i$ must follow consecutive stage transitions i, 2, ..., N to reach the terminate states.   

We seek the optimal control policies $\pi^* = \{\pi_1^*, \pi_2^*, ..., \pi_N^*\}$ via $N$  stage-wise optimal Q functions $Q_1^*, ..., Q_N^*$ such that $\pi_i(s) = \arg\max_a Q_i(s,a)$.
Particularly, we define $Q_i^*$ to be associated with an episodic MDP with modified reward function. The optimal $Q_i^*$ function that starts with initial state $s_0 \in \mathcal{S}_i$ and ends when $s \in \mathcal{S}_{i+1}$ or $s \in \mathcal{S}_T$, given by
\begin{equation}
Q^*_i(s, a) = \max_{\pi_i} \mathbf{E}^\pi[\sum_{t=0}^\infty \gamma^t R_i(s_{t+1}, a_t)|s_0 = s, a_0 = a, \pi_i]
\end{equation}
where $R_i$ is the \textit{modified} reward function for stage $i$. $R_i$ equals to $R$ for all state-action pair $(s, a)$ that does not trigger stage transition and equal to $R$ plus additional stage-transition reward. More formally, let $s_t$ belongs to stage $i$, and the stage reward for transition pair $(s_t, a_t, s_{t+1})$ is defined by 
\begin{equation}
R_i(s_{t+1},a_t) = \begin{cases}
R(s_{t+1}, a_t), & \text{if no stage transition}\\
R(s_{t+1}, a_t) + \gamma V^*_{i+1}(s_t), & \text{if stage transition occurs}
\end{cases}
\end{equation}
where $V^*_{i+1}(s_t)$ is the optimal state value function for stage $i+1$ given by $V^*_{i+1}(s_t) = \max_{a}Q_{i+1}(s_t,a)$. Note that because state transition event will not occur when the system is in the last stage $N$, we have $R_N = R$.
As we mention in the introduction, the introduction of state-transition rewards is the crucial step to ensure all sub-networks/Q functions are trained to  maximize the total reward across all stages instead of the cumulative rewards in its own stage. 

Notably, a MDP with stage structure defined above is \textbf{consistent} to an ordinary MDP in the following way:
\begin{prop}
Consider a $N$-stage MDP. We have 
\begin{itemize}
\item $Q^*_i(s,a) = Q^*(s,a), \forall s\in \mathcal{S}_i, i=1, ..., N.$
\item $\pi^*_i(s) = \pi^*(s), \forall s\in \mathcal{S}_i, i=1, ..., N.$
\end{itemize}
\end{prop}
\begin{proof}
See Appendix.
\end{proof}
 \textbf{Proposition 1} shows that $Q_1^*,..., Q_N^*$ in the stacking architecture is equivalent to $Q^*$ on different portion of $\mathcal{S}$. Therefore, although each sub-network $Q^*_i$ are only responsible for learning and controlling the system at $\mathcal{S}_i$, all of sub-networks collectively learn the optimal control policy.    

\subsection{SDQL architecture and algorithm}

In a typical deep Q learning\cite{Mnih2015Human-levelLearningb}, we use deep neural networks to approximate the optimal Q function, denoted by $Q(s, a; \theta)$, where $\theta$ is the neural network parameters. 
To improve sample efficiency, we usually employ off-policy learning via experience replay buffer $D$\cite{Mnih2015Human-levelLearningb}. In each iteration, we perform gradient descent using a minibatch of samples to update parameters $\theta$ with loss function
\begin{equation}
L_i(\theta_i) =  \mathbf{E}_{(s,a,r,s')\sim U(D)}[(r + \gamma \max_a'Q(s',a';\theta_i^-)-Q(s,a;\theta_i))^2]    
\end{equation}

where the subscript $i$ is iteration index, $(s,a,r,s')$ is the experience tuple uniformly sampled from stored experience $D$, $\theta'$ is the target network parameter that slowly synchronizes with the Q network parameter $\theta$.

In SDQL, as we decompose the control task into $N$ linear stages, we approximate the $Q$ function by $N$ neural networks, denoted by $Q_1(s, a;\theta^{(1)}), ..., Q_N(s, a;\theta^{(N)})$. We also use $N$ experience replay buffers $D_1,..., D_N$, with $D_i$ stores $(s,a,r,s')$ and $I(s) = i$.   We perform gradient descent using a minibatch of samples to update parameters $\theta^{(i)}$ with loss function
$$L^{(i)}(\theta_j^{(i}) =  \mathbf{E}_{(s,a,r,s')\sim U(D_i)}[(r + \gamma \max_a' Q_i(s',a';\theta_j^{-(i)})-Q_i(s,a;\theta_j^{(i)}))^2]$$
where the subscript $j$ is iteration index, $(s,a,r,s')$ is the experience tuple uniformly sampled from the experience buffer $D_i$, $\theta^{-(i)}$ is the target network parameter that slowly synchronizes with the Q network parameter $\theta^{(i)}$. 

Because the reward in each stage depends on the learned optimal value function in subsequent stages. The iterative updates of Q functions should be performed backwards, starting from stage N to stage 1.

The architecture and the training algorithm allows stacking different types of Q-learning networks, including single deep Q network such as DQN\cite{Mnih2015Human-levelLearning}, Dueling-DQN\cite{Wang2015DuelingLearning}, Double-Q DQN\cite{VanHasselt2016DeepQ-learning}, Rainbow\cite{Hessel2018Rainbow:Learning} NAF\cite{Gu2016ContinuousAcceleration}, and actor-critic networks such as DDPG \cite{Lillicrap2015ContinuousLearning} and TD3 \cite{Fujimoto2018AddressingMethods}.

The complete algorithm is showed in \textbf{Algorithm} \autoref{Alg-StackedDQN}.

\begin{figure}
\centering
\includegraphics[height=8cm]{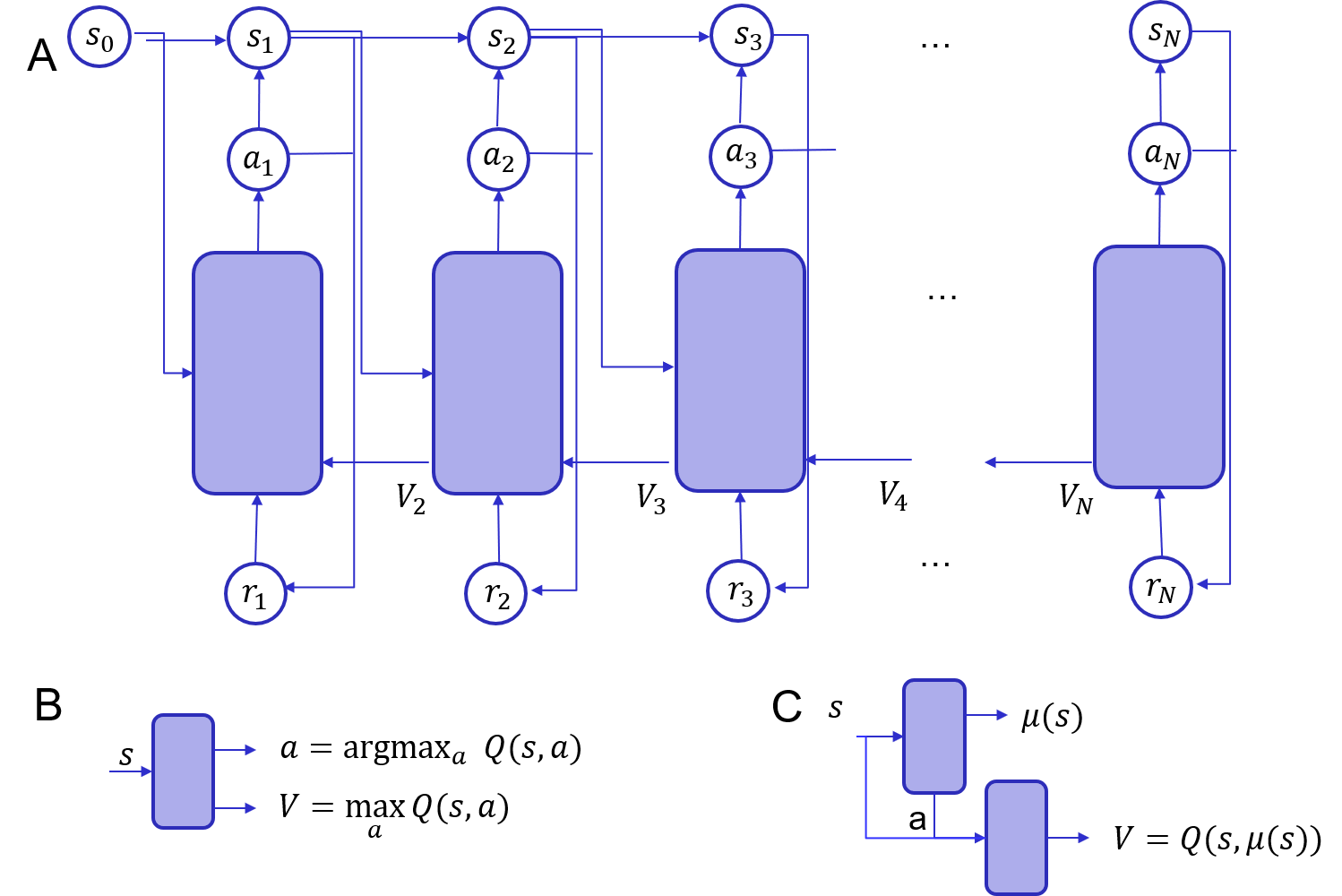}
  \caption{(A) Scheme of SDQL for multi-stage control tasks. Each submodule can be either a Q network (e.g., DON) (B) or a actor-critic network (e.g., DDPG) (C) network that approximates the Q function and the policy for a single stage. Q network produce action $a$ from current
state $s$ to evolve the system state forward and provides value estimate $V$ to train the network in the previous stages.}
  \label{fig:fig1}
\end{figure}

\begin{algorithm}\label{Alg-StackedDQN}
	\caption{Training framework for $N$-stage stacked Q network}
	Initialize $N$ replay memory $D_1,...,D_n$, and the neural networks (e.g., Q network, actor and critic network, and their associated target networks) for each module.\\
    \For{k = N-1, N-2, ..., 0}{
		\For{episode = 1,..., M}{
			Initialize system state to be within the state space $\mathcal{S}_{k} - \mathcal{S}_{k-1}$ and set its stage index to $k$. \\
			\For{ $t = 1,..., maximum\quad step$}{
                select actions $a_t$ using the module networks corresponding to current stage $m$.\\
                update current system state $s_t$ to new system state $s_{t+1}$ based on action $a_t$, observe stage reward $r_t = R_i(s_{t+1}, a_t)$,  \\
                Store transition $(s_t,a_t,r_t,s_{t+1})$ in $D_m$, where $m$ is the stage index of state $s_t$.\\
                
                \For{ i = N, ..., 1}{
                	Sample random minibatch of transitions $(s_j,a_j,r_j,s_{j+1})$ from  $D_i$.\\
                	Set $$y_j = \begin{cases}
                	r_j, & \text{ if episode terminates or state transitions at step} j + 1 \\
                	r_j + \gamma \max_{a'} Q_i(s_{j+1}, a'; \theta^{-(i)}), & \text{otherwise}                
                \end{cases}
              $$
                  Perform a gradient descent step on loss on Q values $(y_j - Q_i(s_j, a_j;\theta^{(i)}))^2$ with respect to $\theta^{(i)}$. Perform a gradient descent step to update actor network if using actor-critic module.\\
                  Gradually synchronize the target networks $\theta^{-(i)}$ with $\theta^{(i)}$.
                }               
			}
		}
	}
\end{algorithm}

\begin{remark}\hfill
\begin{itemize}
    \item We initialize system state to be within the state space $\mathcal{S}_{k} - \mathcal{S}_{k-1}$, instead of the whole $\mathcal{S}_k$ is similar to the initialization in reverse curriculum learning\cite{Florensa2017ReverseLearning}. \\
    \item In \textbf{Proposition 1}, we show that stage-wise Q functions $Q_1^*,...,Q_N^*$ are equivalent to $Q^*$ at different parts of the state space. As we use neural networks to approximate Q functions and the scheduled backward training, in practice, the approximation relation between $Q^*_i$ and $Q^*$ are more accurately described by 
\begin{align}
Q^*_i(s,a) \approx Q^*(s,a),& \forall s \in \mathcal{S}_{i} - \mathcal{S}_{i-1}, i = 2, ..., N \\
Q^*_1(s,a) \approx Q^*(s,a),& \forall s \in \mathcal{S}_{1}
\end{align}
\end{itemize}

\end{remark}

\section{Examples}
\subsection{Multi-stage Grid World navigation}
To demonstrate our algorithm, we first consider a toy program of multi-stage Grid World navigation. As showed in Fig. 2, an agent is within a 25$\times$25 Grid World and aims to navigate from the upper left corner to the target located at lower right corner. The Grid World is spatially divided into five parts, as annotated by the dashed lines, corresponding to the five stages in the whole navigation process. Let $(i, j)$ denote the position of the agent. The agent have four actions, which will enable the agent to move to $(i - 1, j), (i + 1, j), (i, j-1) , (i, j+1)$ in one step, respectively.  If an action causes the agent to move out of the boundary, the agent will not move. The stage transition from one stage to the next stage is triggered when the agent crosses the dashed boundaries. 

We solve this navigation problem by using five DQN modules stacked together. The training process follows the algorithm \ref{Alg-StackedDQN} and starts from the back to the front. We consider two cases to examine how the learned control strategies are adapted to different reward specifications. In the first case, we use homogeneous discount factors, that is, the same $\gamma = 0.99$ for all stages; in the second case, we use heterogeneous discount factors,  $\gamma = 0.99$ for the fourth stage and $\gamma = 0.9$ for other stages. 

The navigation trajectories from the trained multi-stage controller are showed in Fig. 2(b) and (c). For the homogeneous discount case, the agent takes steps to either move down or move right in all stages, which agrees with expectation. As a comparison, as the set a maximum discount factor in the fourth stage in the heterogeneous discount factor case, the agent goes up and right all the way from the bottom to the top at the fourth stage and moves straight right on other stages.  In this way, the agent spent minimal time in all stages except for the fourth, which maximizes the final total reward. To further under how stage-wise $Q$ functions adapt to the specified discount factors, we showed merged state value functions \footnote{$V(s) = V_i(s), \forall s\in \mathcal{S}_{i} - \mathcal{S}_{i-1}, i = 2, ..., N; V(s) = V_1(s), \forall s \in \mathcal{S}_1$} in Fig. 2(d) and (e). For the case of homogeneous discount factors, the value function agrees with the expectation such that state value decreases as distance to the target increases. For the case of heterogeneous discount factors, where the fourth stage has the largest discount factor, the state value decreases as distance to the target increases but decreases at a slower speed when the state is at the fourth stage.

\begin{figure}[H]
\centering
\includegraphics[height=8cm]{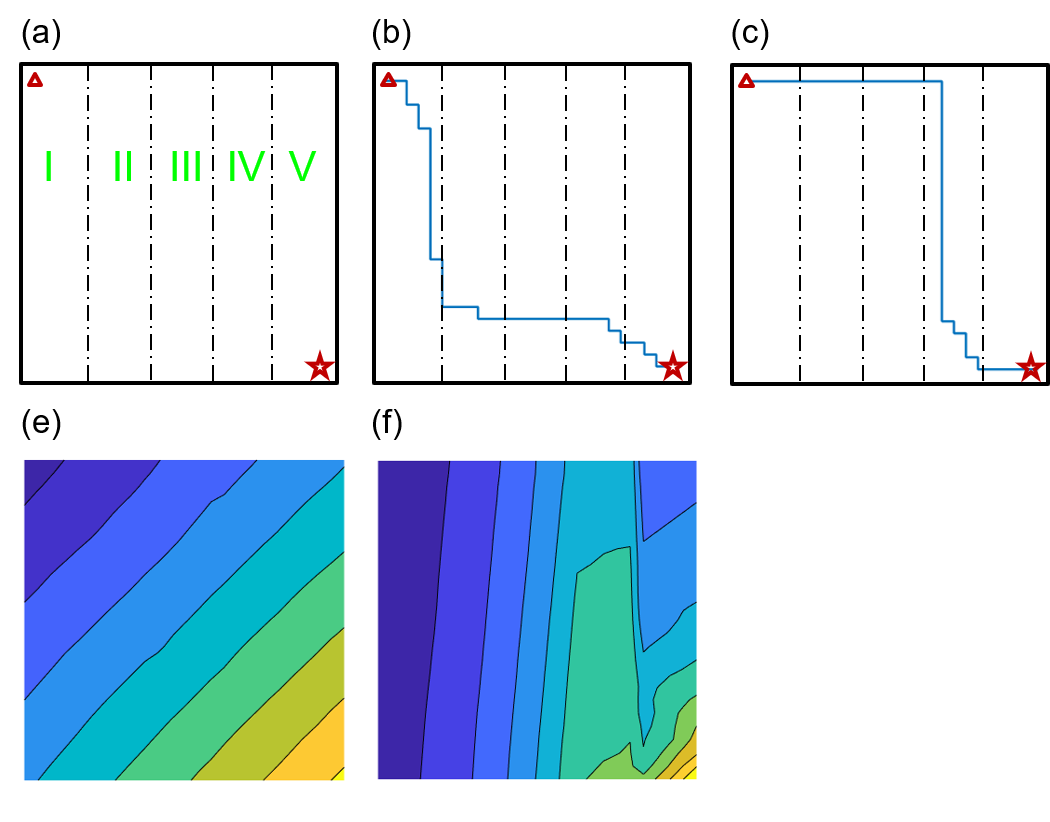}
  \caption{(a) The schematics of a five-stage Grid World navigation task where an agent aims to navigate from the upper left corner (up triangle) to the target  at the lower right corner(star). The Grid World is divided into five parts, annotated by the dashed line, corresponding to the five stages in the navigation process. (b) An example navigation trajectory when trained using homogeneous discount factors. (c) An example navigation trajectory when trained using heterogeneous discount factors. (d) The merged state value function when using homogeneous discount factor. (e) The merged state value function when using heterogeneous discount factor. }
  \label{fig:fig2}
\end{figure}

\subsection{Autonomous cooperative cargo transport}
In the second example, we consider a problem in autonomous targeted nano-drug delivery and precision surgery\cite{Li2017MicroDetoxification, Wang2012Nano/microscaleChallenges} using micro-robots (with size typically ranges from 100nm to 10um). State-of-the-art approaches to targeted drug delivery involve using micro/nano-robots to carry nano-sized drug to hard-to-reach locations (e.g., tissues, cancer cells) and unload the drug there.  Because of their miniaturized size, these robots in general do not carry energy supply by themselves but have the ability to directly utilize energy sources in the environment via many different mechanisms\cite{Mallouk2009PoweringNanorobots, Li2016RocketNanoscale, Yang2018OptimalColloidsb, Yang2019EfficientLearningb} including external electrical and magnetic fields, chemical catalysis, optical field, acoustics, etc. Realistic applications generally requires long-distance travel in a large scale working environment (e.g., tissues, blood vessels) exhibiting strong variations in terms of physical and chemical properties (e.g., temperature, PH, osmotic pressure, etc). The diverse working environments make it infeasible to carry out drug delivery task using one type of micro/nano-robots. One strategy to overcome this limitation is to use multiple micro/nano-robots to cooperatively transport a cargo across different parts of the environment. 

Fig. 3(a) shows a model obstacle environment for multiple micro-robots to realize cooperative cargo transport. The model environment is divided into two parts (dashed line), corresponding the two stages in the transport. In the first stage, one robot is responsible to carry the cargo and unload at the border; in the second stage, the robot is responsible to continue with where the first robot has left and carry the cargo to the final specified target site. We also assume the two micro-robots are functioned on different transport mechanisms and thus have different equations of motion. 
Specifically, the equation of motion of the robot at the first and second stage are given by 
\begin{align}
\begin{split}
    \partial_t x &= v\cos(\theta) + \xi_x(t)\\
    \partial_t y &= v\sin(\theta) + \xi_y(t) \\
    \partial_t \theta & = \xi_\theta(t)
\end{split}
\end{align}
and 
\begin{align}
\begin{split}
    \partial_t x &= v_max\cos(\theta) + \xi_x(t)\\
    \partial_t y &= v_max\sin(\theta) + \xi_y(t) \\
    \partial_t \theta & = w + \xi_\theta(t)
\end{split}
\end{align}
where $x$, $y$, and $\phi$ denote the robot's position and orientation, $w$ in the first equation of motion is the control input that controls the self-propulsion direction. $w$ takes continuous value between $[-1, 1]$. $v$ in the second equation of motion is the control input that controls the self-propulsion speed . $v$ takes binary values of $0$ and $v_max$, where $v_max$ is the maximum allowed speed. $\xi_x, \xi_y, \xi_\theta$ are independent zero mean white noise stochastic process satisfying $E[\xi_\theta(t)\xi_{\theta}(t')] = 2D_r\delta(t-t')$ and $E[\xi_x(t)\xi_x(t')] = E[\xi_y(t)\xi_y(t')] = 2D_t\delta(t-t')$, and $D_t$ and $D_r$ are transnational and rotational diffusivity coefficients.  Compared the traditional human-size robots, micro/nano-robots are usually under-actuated (i.e., not all degrees of freedom can be controled) and subject to Brownian motion.

We refer to the robot in the first stage as slider robot because it can directly control its moving direction but not its propulsion speed. We also refer to the robot in the second stage as the self-propelled robot because it can only control its propulsion speed but not its orientation. It has been showed in the previous studies\cite{Yang2018OptimalColloidsb, Mano2017OptimalSteering} that, a self-propelled micro-robot has to smartly exploit Brownian rotation and self-propulsion decisions to realize efficient transport, while a constantly self-propelled simply reduce to an effective random walk on longer time scale as its orientation fully randomizes\cite{Howse2007Self-MotileWalk}. In general, the slider robot have better mobility and we expect it will contribute more in the cooperative transport process.

Our multi-stage controller consists of a TD3 module \cite{Fujimoto2018AddressingMethods} and a DQN module stacked together. For a given micro-robot's state at step $n$, denoted by $s_n = (x_n, y_n, \theta_n)$, the observation $\phi(s_n)$ at $s_n$ consists of pixel-level image of $10\times 10$ around the micro-robot and the target's coordinate in the micro-robot's own coordinate system.

Fig. 3(b) shows the example trajectories of the two micro-robots transporting within the obstacle environment. As expected, the first slider robot will travel longer distance and finish its travel close to the target. The self-propelled robot will continue at where slider ends and finish the rest of the transport. In this way, the two robots can cooperatively transport the cargo in minimum time.
  
\begin{figure}[H]
\centering
\includegraphics[height=4cm]{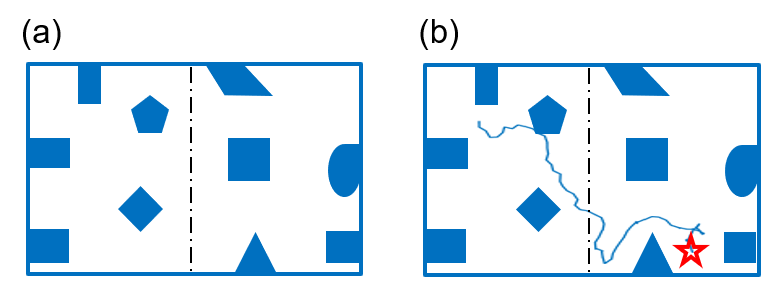}
  \caption{(a) The schematics of the cooperative cargo transport problem in an obstacle environment that consists of two stages of transport processes (divided by the dashed line). The first stage involves steering a slider micro/nano-robot to the second part of the environment and ends when the micro-robot crosses the dashed line. The second stage involves a self-propelled micro/nano-robot starting at the end state of the first stage and continue navigation towards the target (denoted by a star). (b) An example controlled trajectory of the two-stage cooperative cargo transport using SDQL.}
  \label{fig:fig3}
\end{figure}

\subsection{Precision manipulator control}

We now demonstrate SDQL can also be used to compute control policies for tasks involving multi-scale dynamics. Consider the task of controlling a two-arm manipulator (\ref{fig:fig4}) to reach a prescribed target with high precision. Specifically, we denote the length of the two arms by $L_1 = 1$ and $L_2 = 1$, the target by $(x_t, y_t)$, the position of the end effector by $(x_2, y_2)$. The goal is achieved if the effector-target distance $d \triangleq \sqrt((x_t-x_2)^2 + (y_t-y_2)^2) \leq 0.01$.
The geometry relationship (\ref{fig:fig4}) gives the following relationship: 
\begin{align}
\begin{split}
    x_1 &= L_1 \cos(\theta_1) \\
    y_1 &= L_1 \sin(\theta_1) \\
    x_2 &= x_1 + L_2 \cos(\theta_1 + \theta_2) \\
    y_2 &= y_1 + L_2 \sin(\theta_1 + \theta_2) 
\end{split}
\end{align}
The control on the arm angle $\theta_1, \theta_2$ can be formulated as a discrete-time dynamic model given by 
\begin{equation}
  \theta_1(t + \Delta t) = \theta_1 + u_1, \theta_2(t + \Delta t) = \theta_2 + u_2,  
\end{equation}
where $\Delta t$ is the control update time interval, $u_1, u_2 \in [-u_max, u_max]$, and $u_max = 0.25$ is the maximum allowed angle movement in one step. 

We aim to obtain a control strategy that can move the two-arm configuration from an arbitrary initial configuration to a configuration such that the end effector is sufficiently close to the target position.
To make our test more general, we only use the natural reward where the reward is 1 if $d \leq 0.01$ and 0 otherwise.

Therefore, there exist two length scales separated in the control task: a large length scale when the end effector is far away from the target and large step size to get to the target; a small length scale when the end effector is near the target and small steps are preferred for fine adjustments.

Using single length scale via deep reinforcement learning can cause sparse reward issue (reaching the target is quite rare) if the length scale used is relatively large or cause temporally extended delayed reward (reaching the target takes considerably more steps) in the length scale used is relatively small. The length-scale separation can be addressed by formulating the control task into a two-stage control task that uses different length scale in each stage. We introduce the two-stage stage structure in the following way. The first stage state space is characterized by $d>= 0.2$ and the transition from the first stage to the second stage is triggered by $d < 0.2$. In the second stage, we set the maximum allowed action by 0.05 and set the normalizing factor for the distance by 0.05. The use of smaller action can also be interpreted as we are evolving the dynamics at a proportionally smaller time scale. 

For the two-stage control task, we stack two DDPG actor-critic in the SDQL architecture to seek the optimal control policy. The input state is the concatenation of $(x_1,y_1)$, $(x_2,y_2)$ and target distance to the end effector $((x_t-x_2), (y_t-y_2))$.

By using a small length scale in the second stage, we increase the probability that the agent can receive positive signal in the second stage during its exploration and exploitation process. The learned value function in the second stage then can propagate towards to the first stage Q network via bootstrapping.

Fig. 4(b) shows the example trajectories of controlling the two-arm robot, starting from randomly initialized positions,  to reach target position (1,1) with high precision.

\begin{figure}[H]
\centering
\includegraphics[height=6cm]{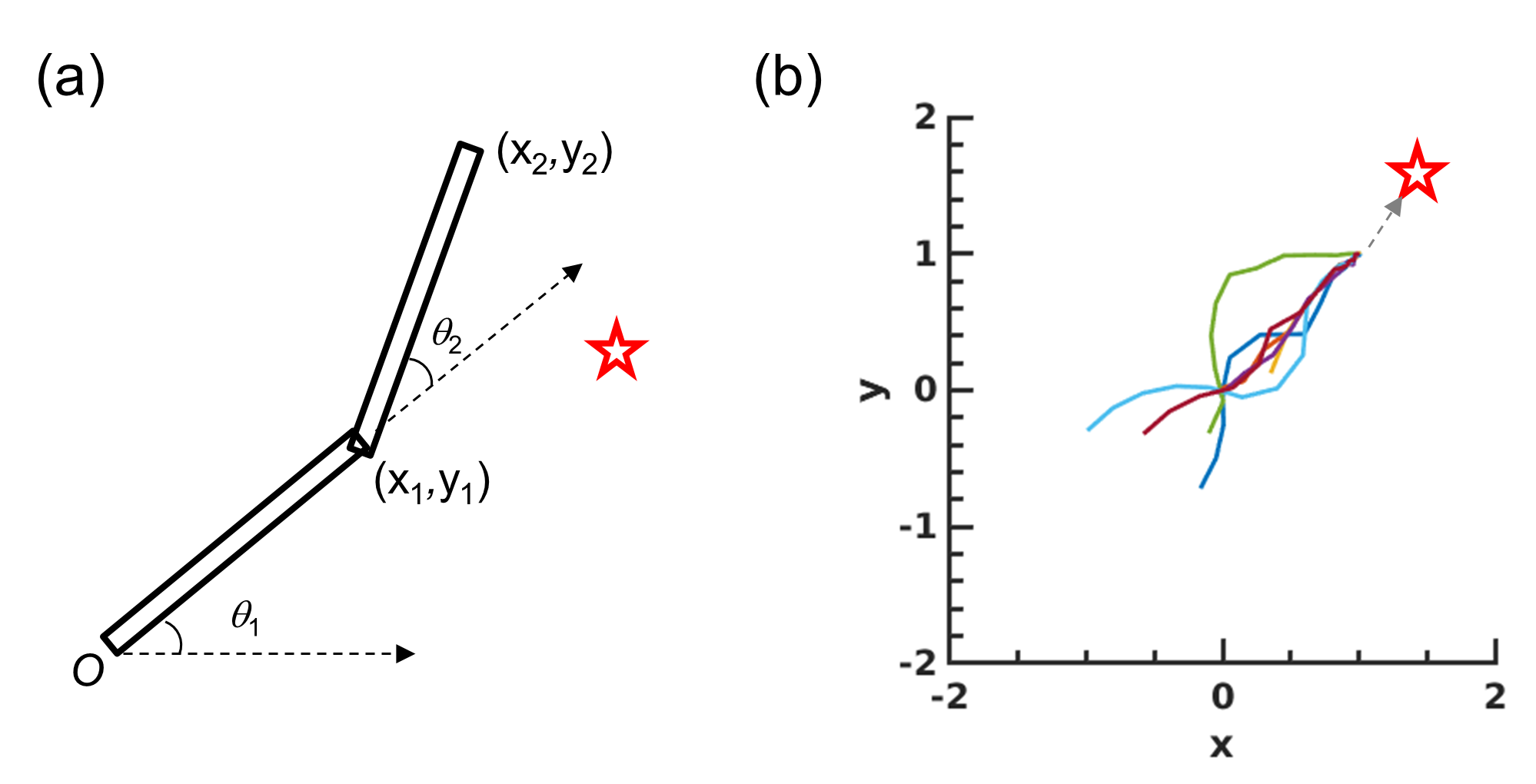}
  \caption{(a) The schematics of a two-arm robot manipulator control task that aims to control the manipulator to reach the target denoted by the star. (b) Representative trajectories of the end effector position moving towards the target located at (1,1). }
  \label{fig:fig4}
\end{figure}

\section{Related Work}

End-to-end deep reinforcement learning have recently achieved remarkable performance in various domains includes games\cite{Mnih2015Human-levelLearningb}\cite{Silver2016MasteringSearch}, robotics\cite{Levine2015End-to-EndPolicies}, bio-mechanics\cite{Verma2018EfficientLearning}, medicine\cite{Popova2018DeepDesign}, physics, and chemical science\cite{Zhou2017OptimizingLearning} etc. For multi-stage control tasks,  although it is possible to use one end-to-end large neural network that simultaneously obtain all policies for each stage, the training of such a giant complex neural network for control tasks with high-dimensional state space and sparse rewards can be considerably brittle due to the deadly triad of nonlinear function approximator, bootstrapping, and off-policy learning \cite{Achiam2019TowardsQ-Learning}. By introducing stage structure into single stage control task or exploits the existing stage structure, SDQL enables decomposition of long horizon complex task into smaller one, therefore stabilize the learning process compared to training a giant complex network.

The stacking structure and the backward training scheme resembles the reverse curriculum learning\cite{Bengio2009CurriculumLearning, Florensa2017ReverseLearning, Matiisen2017Teacher-StudentLearning}. Curriculum learning usually employs a single network for the whole task and schedules the learning of tasks of different difficulties to accelerate the training process. During learning, tasks of different levels have to be sampled constantly and at increasing difficult levels to accerlate training and prevent to agent forget skills already learned. In SDQL, multiple small neural network are responsible for learning different skills, starting with easier ones (i.e., later stage), thus improving the stability in the training and avoid the forgetting problem.

Finally, SDQL can be viewed as a regularized version of hierarchical reinforcement learning \cite{Bacon2017TheArchitecture}\cite{Frans2017MetaHierarchies}\cite{Kumar2018Robotics}\cite{Nachum2018Data-EfficientLearning}. Hierarchical reinforcement learning has to learn hierarchy and temporal abstraction from data, in the form of high-level policy (meta controller), to select subgoals for the low level controller to achieve. SDQL, instead, directly exploits a given hierarchical structure.

\section{conclusion}

We have introduced SDQL, a flexible stacking architecture that enables application of deep Q learning to general multistage optimal control tasks. Our theoretical analysis shows h the decomposition of Q values in the stacking architecture can be achieved in a consistent way. A general off policy training algorithm is developed to train the stacked networks via backward state-value propagation.

We used different examples to demonstrate the strategic application of SDQL to learn competitive strategies in problem with characteristics of high-dimensional state space, heterogeneous action space(both discrete and continuous), multiple scales, and sparse and delayed rewards.

There can be different modification to our architecture.
For example, different Q network can share the parameters of first several layer. Different enhancement to DON and DDPG can be readily applied. Depending on the application, parameters of later stage networks can be copied to early stage network as pertained
parameters.

\appendix

\section{Proof of Proposition 1}
\begin{proof}
We use mathematical induction to prove the proposition.

First consider the case of $I(s_0) = N$. Based on the definition of stage reward function $R_N = R$, we have 
\begin{align*}
Q^*(s, a) &= \max_\pi \mathbf{E}^\pi[\sum_{t=0}^\infty \gamma^t R(s_{t+1}, a_t)|s_0 = s, a_0 = a, \pi] \\
&= \max_\pi \mathbf{E}^\pi[\sum_{t=0}^\infty \gamma^t R_N(s_{t+1}, a_t)|s_0 = s, a_0 = a, \pi] \\
&= Q^*_N(s,a)
\end{align*}
Similarly, we have $\pi^* = \pi^*_N$.

Suppose $Q^*_i(s,a) = Q^*(s,a), \forall s\in \mathcal{S}_i, i= K, ..., N.$ and $\pi^*_i(s) = \pi^*(s), i=1, ..., N.$.
Consider $I(s_0) = K - 1$, and stage transition occurs at step $M$ we have
\begin{align*}
Q^*(s, a) &= \max_\pi \mathbf{E}^\pi[\sum_{t=0}^\infty \gamma^t R(s_{t+1}, a_t)|s_0 = s, a_0 = a, \pi] \\
&=\max_\pi \mathbf{E}^\pi[\sum_{t=0}^{M-1} \gamma^t R(s_{t+1}, a_t) +\sum_{t=M}^\infty \gamma^t R(s_{t+1}, a_t) |s_0 = s, a_0 = a, \pi, I(s_{M})=K] \\
&=\max_\pi \mathbf{E}^\pi[\sum_{t=0}^{M-1} \gamma^t R(s_{t+1}, a_t) +\gamma^M \sum_{t=0}^\infty \gamma^{t} R(s_{t+M + 1}, a_t) |s_0 = s, a_0 = a, \pi, , I(s_{M})=K] \\
&=\max_\pi \mathbf{E}^\pi[\sum_{t=0}^{M-1} \gamma^t R(s_{t+1}, a_t) + |s_0 = s, a_0 = a, \pi] + \gamma^M V^*_{K}(s_{M}) \\
&=\max_\pi \mathbf{E}^\pi[\sum_{t=0}^{K-1} \gamma^t R(s_{t+1}, a_t) |s_0 = s, a_0 = a, \pi] + \gamma^M V^*_{K}(s_{M}) \\
&=\max_\pi \mathbf{E}^\pi[\sum_{t=0}^{M} \gamma^t R_{K-1}(s_t, a_t) |s_0 = s, a_0 = a, \pi]  \\
& = Q^*_{K-1}(s, a)
\end{align*}
where we use the fact that
$$\max_\pi E^\pi[\sum_{t=M}^\infty \gamma^t R(s_{t+1}, a_t) |s_M, , I(s_{M})=K]=V^*(s_M) $$
and the definition of $R_{K-1}$. 
\end{proof}
\medskip
\bibliographystyle{unsrt}  
\bibliography{references}
\end{document}